\documentclass[conference,letterpaper]{IEEEtran}

\usepackage[utf8]{inputenc} 
\usepackage[T1]{fontenc}
\usepackage{url}
\usepackage{ifthen}
\usepackage{cite}
\usepackage[cmex10]{amsmath} % Use the [cmex10] option to ensure complicance
\usepackage{color}
\usepackage{bm}
\usepackage{graphics}
\usepackage{mathtools}
\usepackage{amsmath,amssymb,amsfonts} % for argmin

\usepackage{booktabs} % tabular
\usepackage{multirow}
\usepackage{hyperref}
\usepackage{mathtools}
\usepackage{comment}

\DeclarePairedDelimiter{\abs}{\lvert}{\rvert}

\DeclareMathOperator*{\argmin}{arg\,min}

\newtheorem{theorem}{Theorem}
\newtheorem{proof}{Proof}
\newtheorem{definition}{Definition}
\newcommand{\qed}{\hfill$\Box$}%{\begin{flushright}$\Box$\end{flushright}}
\newtheorem{assumption}{Assumption}

\begin{document}
\title{Unbiased Estimation Equation under $f$-Separable Bregman Distortion Measures} 

% %%% Single author, or several authors with same affiliation:
% \author{%
%   \IEEEauthorblockN{Stefan M.~Moser}
%   \IEEEauthorblockA{ETH Zürich\\
%                     ISI (D-ITET)\\
%                     CH-8092 Zürich, Switzerland\\
%                     Email: moser@isi.ee.ethz.ch}
% }

%%% Several authors with up to three affiliations:
\author{%
  \IEEEauthorblockN{Masahiro Kobayashi and Kazuho Watanabe}
  \IEEEauthorblockA{
  %%%Department of Computer Science and Engineering\\
                    Toyohashi University of Technology\\
                    Email: kobayashi@lisl.cs.tut.ac.jp and wkazuho@cs.tut.ac.jp}
  %\and
  %\IEEEauthorblockN{Kazuho Watanabe}
  %\IEEEauthorblockA{Department of Computer Science and Engineering\\
   %                 Toyohashi University of Technology\\ 
    %                Email: wkazuho@cs.tut.ac.jp}
}

%%% Many authors with many affiliations:
% \author{%
%   \IEEEauthorblockN{Albus Dumbledore\IEEEauthorrefmark{1},
%                     Olympe Maxime\IEEEauthorrefmark{2},
%                     Stefan M.~Moser\IEEEauthorrefmark{3}\IEEEauthorrefmark{4},
%                     and Harry Potter\IEEEauthorrefmark{1}}
%   \IEEEauthorblockA{\IEEEauthorrefmark{1}%
%                     Hogwarts School of Witchcraft and Wizardry,
%                     1714 Hogsmeade, Scotland,
%                     \{dumbledore, potter\}@hogwarts.edu}
%   \IEEEauthorblockA{\IEEEauthorrefmark{2}%
%                     Beauxbatons Academy of Magic,
%                     1290 Pyrénées, France,
%                     maxime@beauxbatons.edu}
%   \IEEEauthorblockA{\IEEEauthorrefmark{3}%
%                     ETH Zürich, ISI (D-ITET), ETH Zentrum, 
%                     CH-8092 Zürich, Switzerland,
%                     moser@isi.ee.ethz.ch}
%   \IEEEauthorblockA{\IEEEauthorrefmark{4}%
%                     National Chiao Tung University (NCTU), 
%                     Hsinchu, Taiwan,
%                     moser@isi.ee.ethz.ch}
% }

\maketitle

%%%%%%
%% Abstract: 
%% If your paper is eligible for the student paper award, please add
%% the comment "THIS PAPER IS ELIGIBLE FOR THE STUDENT PAPER
%% AWARD." as a first line in the abstract. 
%% For the final version of the accepted paper, please do not forget
%% to remove this comment!
%%
\begin{abstract}
We discuss unbiased estimation equations in a class of objective function using a monotonically increasing function $f$ and Bregman divergence.
The choice of the function $f$ gives desirable properties such as robustness against outliers.
%Generally, however, it does not guarantee the consistency of the estimator because it does not hold an unbiased estimation equation.
In order to obtain unbiased estimation equations, analytically intractable integrals are generally required as bias correction terms.
In this study, we clarify the combination of Bregman divergence, statistical model, and function $f$ in which the bias correction term vanishes.
%%%Specifically, we investigate the cases of Mahalanobis and Itakura-Saito distances.
Focusing on Mahalanobis and Itakura-Saito distances,
%%%In the case of Mahalanobis distance,
we provide a generalization of fundamental existing results and characterize
%%%In the case of the Itakura-Saito distance,
a class of distributions of positive reals with a scale parameter,
%%%is characterized by the Itakura-Saito distance,  
which includes the gamma distribution as a special case.
We discuss the possibility of latent bias minimization when the proportion of outliers is large, which is induced by the extinction of the bias correction term.
%Further, we apply the estimation method to regression models and demonstrate the robustness against outliers via experiments.
\end{abstract}
%\textit{A full version of this paper is accessible at:}
%\url{http://lisl.cs.tut.ac.jp/~kobayashi/pdf/kobayashi2020.pdf}

%%%\vspace{-10pt}
%% The paper must be self-contained. However, if you are referring to
%% a full version for checking certain proofs, please provide the
%% publically accessible location below.  If the paper is completely
%% self-contained, you can remove the following line from your
%% submission.
\section{Introduction}
%Statistical estimation is to estimate true distribution which is data-generating distribution from given the data.
The maximum likelihood estimation (MLE) for the statistical model $p(\bm{x}|\bm{\theta})$ estimates the parameter $\bm{\theta}$ by minimizing the negative log-likelihood.
It is equivalent to empirical inference under the Kullback-Leibler (KL)-divergence.
%%%from the \textcolor{blue}{data-generating} distribution to the statistical model $p(\bm{x}|\bm{\theta})$.
However, MLE is susceptible to outliers or mismatch of the assumed model.
In robust statistics, estimation methods weakening adverse effect of outliers have been studied \cite{robust}, \cite{huberbook}.
One of the most popular methods is M-estimation
%%%which changes negative log-likelihood function $-\sum_{i=1}^n \log p(\bm{x}_i|\bm{\theta})$ in MLE to a general robust objective function \cite{robust}, \cite{huberbook}. %%%$\sum_{i=1}^n \rho(\bm{x}_i, \bm{\theta})$ \cite{robust}, \cite{huberbook}.
\begin{comment}
In the view of MLE, it is equivalent to assuming another data-generating distribution.
An alternative definition is the root $\bm{\theta}$ of the estimation equation $\sum_{i=1}^n \psi(\bm{x}_i,\bm{\theta})=\bm{0}$.
Usually, $\psi$ function satisfies an unbiased estimation equation $\mathbb{E}_{p(\bm{x}|\bm{\theta})}[\psi(\bm{x},\bm{\theta})]=\bm{0}$.
In many cases, these satisfy $\frac{\partial \rho(\bm{x}, \bm{\theta})}{\partial \bm{\theta}}=\psi(\bm{x},\bm{\theta})$ and become an identical estimation method.
To distinguish, the latter method is specifically called Z-estimation \cite{asympt}.
\end{comment}
%%%Recently, there have been developed M-estimation methods 
which changes KL-divergence corresponding to MLE to robust divergences applicable to empirical inference.
These divergences are constructed through estimation equations by weighted (negative) score function $s(\bm{x},\bm{\theta})=\frac{\partial l(\bm{x},\bm{\theta})}{\partial \bm{\theta}}$, where $l(\bm{x}, \bm{\theta})=-\log p(\bm{x}|\bm{\theta})$.
The following two types of estimation equations are well known:
\begin{align}
\frac{1}{n}\sum_{i=1}^n \xi(l(\bm{x}_i,\bm{\theta}))s(\bm{x}_i,\bm{\theta})=  \mathbb{E}_{p(\bm{x}|\bm{\theta})}\left[\xi(l(\bm{x},\bm{\theta}))s(\bm{x},\bm{\theta})\right], \label{eq:ueq}\\
%\end{align}
%\begin{align}
%\MoveEqLeft
\frac{\sum_{i=1}^n \xi(l(\bm{x}_i,\bm{\theta}))s(\bm{x}_i,\bm{\theta})}{\sum_{j=1}^n \xi(l(\bm{x}_j,\bm{\theta}))}= \frac{\mathbb{E}_{p(\bm{x}|\bm{\theta})} \left[\xi(l(\bm{x},\bm{\theta}))s(\bm{x},\bm{\theta})\right]}{\mathbb{E}_{p(\bm{x}|\bm{\theta})}\left[ \xi(l(\bm{x},\bm{\theta}))\right]}, \label{eq:neq}
\end{align}
where $\xi:\mathbb{R}\to\mathbb{R}$ works as the weight function. %\textcolor{blue}{and is the derivative of the function $\tilde{f}:\mathbb{R}\to\mathbb{R}$.}
%\textcolor{blue}{However, if estimation problem is formulated  by \eqref{eq:ueq} or \eqref{eq:neq}, the function $\tilde{f}$ may not exist. }
%\textcolor{blue}{The weighted function of \eqref{eq:ueq} and \eqref{eq:neq} are expressed as $\tilde{f}'$ for distinguish them from the function $f:\mathbb{R}_+\to\mathbb{R}$ of $f$-separable Bregman distortion measures and its derivative $f'$, which is adopted in this paper for later use.}
%\sout{The references on the (un) normalized estimation equations denote the weight function by $\xi$ instead of $f'$, which is adopted in this paper for later use.}
Equation \eqref{eq:ueq} is called the {\it unnormalized estimation equation} because the summation of weights of score functions is not one.
This estimation equation is obtained from minimizing $\beta$-divergence (density power divergence), $U$-divergence, $\Psi$-divergence and so on \cite{beta-div, u-boost, psi-div, bexp-div}.
Equation \eqref{eq:neq} is called the {\it normalized estimation equation} because the summation of weights of score functions is one.
Windham proposed the estimator using density power weight in \eqref{eq:neq} \cite{windham}.
Then Jones et al. constructed corresponding divergence \cite{type0-div}.
It was proved that this divergence, named  $\gamma$-divergence, has the property that the latent bias can be minimized even when the proportion of outliers is large, and that the divergence with such a property is unique under some assumptions \cite{gamma-div}.
%The property of the divergence was elucidated which can minimize latent bias when the proportion of outliers is large by \cite{gamma-div}.
%They called it $\gamma$-divergence.
%Simultaneously, they proved such divergence is only one under some assumption.
This property of $\gamma$-divergence was extended to the normalized estimation equation \eqref{eq:neq} with general weight $\xi$ \cite{fujisawa2013}.
However, these approaches require bias correction terms, that is, the right hand sides of \eqref{eq:ueq} and \eqref{eq:neq}, which in general result in analytically intractable integrals.

In this paper, we consider the M-estimation under $f$-separable distortion measures, which were proposed to extend linear distortion such as the average distortion to non-linear distortion, and for which 
the rate-distortion function
%, which indicates the limit of lossy compression,
was studied %%%for $f$-separable distortion measures 
\cite{f-separable}.
%%%In the context of clustering, 
It was also applied %%%$f$-separable distortion measures
to the estimation problem with Bregman divergence as the base distortion measure and a simple clustering  or vector quantization algorithm was constructed \cite{mypaper3}.
In this paper, we call this class of objective functions the $f$-separable Bregman distortion measure.
%%%The estimators obtained by minimizing $f$-separable Bregman distortion measures are characterized by the function $f$.
%%%If the function $f$ is concave, the estimator is shown to be robust against outliers.
%%%However, without the bias correction term, the estimation equation is not always unbiased and the consistency of the estimator does not hold in general.
As will be discussed in Section \ref{sec:robust_div}, the M-estimation under this distortion measure can be viewed as deviance-based estimation of the regular exponential family model.
On one hand, unbiasedness of the estimation equation of deviance-based methods has been studied and some sufficient conditions for it have been obtained \cite{deviance-based},\cite{bianco2005}.
However, these results only apply to the case where the data-generating distribution is included in the assumed model.
On the other hand, the M-estimation of the location family is proved to have an unbiased estimation equation for general symmetric distributions \cite{huberbook}.
It is unknown in what cases of $f$-separable Bregman distortion measures the estimation equation is unbiased for such a general class of distributions.
%%%In order to satisfy unbiased estimation equations, analytically intractable integrals are generally required as bias correction terms.
If an estimation equation is unbiased, it can be regarded as normalized and the estimator has the potential to minimize the latent bias even if the proportion of outliers is large.

In this paper, we study the conditions for bias correction terms of $f$-separable Bregman distortion measures to vanish and characterize the combination of Bregman divergence, the statistical model, and the function $f$.
Focusing on Mahalanobis and Itakura-Saito (IS) distances, we specify the conditions for the general model classes and the function $f$ to achieve unbiased estimation equations.
Furthermore, we discuss if the latent bias can be minimized when the proportion of outlier is large.
%%%Focusing on the case of IS distance, 
We compare the M-estimation under the $f$-separable IS distortion measure with the estimation methods minimizing $\beta$ and $\gamma$ divergences in terms of asymptotic efficiency.
\section{$f$-separable Bregman distortion measures \label{sec:f-separable}}
In this section, we introduce the estimation method based on $f$-separable Bregman distortion measures \cite{mypaper3}.
We consider estimating the parameter $\bm{\theta}\in\bm{\Theta}\subseteq\mathbb{R}^d$ of a statistical model $p(\bm{x}|\bm{\theta})$ when given the data $\bm{x}^n=\{\bm{x}_1, \cdots, \bm{x}_n\}$, $\bm{x}_i=(x_i^{(1)},\cdots,x_i^{(d)})^{\rm T}\in\mathbb{R}^d$.
We assume that $p(\bm{x}|\bm{\theta}^*)$ is the data-generating distribution and the parameter $\bm{\theta}$ is the expected value of $\bm{x}$ under the model, that is, $\bm{\theta}=\mathbb{E}[\bm{X}]=\int \bm{x}p(\bm{x}|\bm{\theta})d\bm{x}$ if it exists.
The objective function \eqref{eq:f-separable} is defined by a differentiable and continuous monotonically increasing function $f : \mathbb{R}_+\to\mathbb{R}$ and Bregman divergence $d_\phi(\bm{x}, \bm{\theta}): \mathbb{R}^d \times \mathbb{R}^d \to \mathbb{R}_+$, where $\mathbb{R}_+$ is the set of non-negative real numbers.
\begin{align}
L_f(\bm{\theta})=\frac{1}{n}\sum_{i=1}^n f\left(d_\phi\left(\bm{x}_i, \bm{\theta}\right)\right) \label{eq:f-separable}
\end{align}
Bregman divergence is defined by a differentiable strictly convex function $\phi: \mathbb{R}^d \to \mathbb{R}$ as
\begin{align*}
d_\phi(\bm{x}, \bm{\theta}) \triangleq \phi(\bm{x}) - \phi(\bm{\theta}) - \langle \bm{x}-\bm{\theta}, \nabla\phi(\bm{\theta})\rangle ,
\end{align*}
where $\nabla \phi$ is its gradient vector and $\langle \cdot,\cdot\rangle$ is the inner product.
The estimator $\hat{\bm{\theta}}$ of the parameter $\bm{\theta}^*$ is given by the minimum solution of \eqref{eq:f-separable} as
\begin{align*}
\hat{\bm{\theta}} = \argmin_{\bm{\theta}} L_f(\bm{\theta}) .
\end{align*}
The corresponding estimation equation is given by
\begin{align}
\frac{1}{n}\sum_{i=1}^n f'(d_\phi(\bm{x}_i, \bm{\theta}))\frac{\partial}{\partial \bm{\theta}}d_\phi(\bm{x}_i, \bm{\theta}) = \bm{0}, \label{eq:f-sep_eq}
\end{align}
where $f'$ is the derivative of $f$.
This is not generally unbiased.
The property of the estimator depends on the function $f$.
For example, if the function $f$ is concave, the estimator is robust against outliers.
The original $f$-separable distortion measures are defined by $f$-mean with respect to some base distortion $d$ \cite{f-separable}.
From the view point of $f$-mean, representative examples are the log-sum-exp function and power mean, which are given by the following functions:
\begin{align}
f(z)&=\frac{1-\exp(-\alpha z)}{\alpha}, f'(z)=\exp(-\alpha z)\label{eq:lse} \\
f(z)&=\frac{(z+a)^\beta-1}{\beta}, f'(z)=(z+a)^{\beta-1}\quad (a\geq0), \label{eq:pow}
\end{align}
respectively, where if tuning parameters satisfy $\alpha>0$ or $\beta<1$, the estimators become robust.
When $\alpha=0$ and $\beta=1$, \eqref{eq:lse} and \eqref{eq:pow} become linear functions.
%In this paper, we focus \eqref{eq:lse}.

\section{Relation to robust divergences}
\label{sec:robust_div}
%If the unnormalized estimation equation \eqref{eq:ueq} can be integrated with respect to the parameter $\bm{\theta}$, the corresponding pseudo cross entropy (objective function) is written by
%\begin{align*}
%\frac{1}{n}\sum_{i=1}^n \Xi(l(\bm{x}_i,\bm{\theta}))-b_{\Xi}(\bm{\theta}), %\label{eq:loss_score}
%\end{align*}
%where $b_\Xi(\bm{\theta})$ is the bias correction term.
%%%Table \ref{table:div_list} shows divergences corresponding to the weight functions $\xi$.

%%%In this section, we assume that the statistical model is regular exponential family and consider the problem of estimating the expectation parameter.
First, we show that the minimization of $L_f(\bm{\theta})$ is derived from deviance-based M-estimation of the expectation parameter under the regular exponential family,
%Then, the regular exponential family can be written by the Bregman divergence as follows: 
\begin{align}
p(\bm{x}|\bm{\theta})=r_\phi(\bm{x}) \exp(-d_\phi(\bm{x},\bm{\theta})) , 
%-\log p(\bm{x}|\bm{\theta})=d_\phi(\bm{x},\bm{\theta})-\log r_\phi(\bm{x}),
\label{eq:exp_fam}
\end{align}
where $r_\phi(\bm{x})$ is uniquely determined by the strictly convex function $\phi$ \cite{clustering-bregman}.
In fact, the deviance function \cite{deviance-based} of this model is
\begin{align*}
l(\bm{x},\bm{\theta})-\inf_{\bm{\theta}}l(\bm{x},\bm{\theta})= d_\phi(\bm{x},\bm{\theta})-\min_{\bm{\theta}}d_\phi(\bm{x},\bm{\theta})= d_\phi(\bm{x},\bm{\theta}).
\end{align*}

Next, we turn to empirical inference based on robust divergences under the regular exponential family \eqref{eq:exp_fam}.
The negative score function is given by $s(\bm{x}, \bm{\theta})=\frac{\partial }{\partial \bm{\theta}} d_\phi(\bm{x}, \bm{\theta})$.
Suppose for a moment that the bias correction term can be ignored.
In this case, the unnormalized estimation equation \eqref{eq:ueq} is given by
\begin{align}
\begin{split}
\frac{1}{n}\sum_{i=1}^n \xi\left(l(\bm{x}_i,\bm{\theta})\right)\frac{\partial }{\partial \bm{\theta}} d_\phi(\bm{x}_i, \bm{\theta})=\bm{0}.
%%%\bm{0}&=\frac{1}{n}\sum_{i=1}^n \xi\left(l(\bm{x}_i, %%%\bm{\theta})\right)s(\bm{x}_i, \bm{\theta}) 
%%%  \\&=\frac{1}{n}\sum_{i=1}^n \xi\left(d_\phi(\bm{x}_i, \bm{\theta})-\log r_\phi(\bm{x}_i) \right)\frac{\partial }{\partial \bm{\theta}} d_\phi(\bm{x}, \bm{\theta}).
\end{split} \label{eq:bphi}
\end{align}
%%%Comparing the estimation equations \eqref{eq:f-sep_eq} and \eqref{eq:bphi}
%%%, when we assume that the statistical model is regular exponential family, 
Compared with this estimation equation,
the estimation equation \eqref{eq:f-sep_eq} of $f$-separable Bregman distortion measures can be interpreted as a weighted score function.
We focus on the arguments of  the weight functions of \eqref{eq:f-sep_eq} and \eqref{eq:bphi}.
The only difference is the term $\inf_{\bm{\theta}}l(\bm{x},\bm{\theta})=-\log r_\phi(\bm{x})$. Specifically, if the domain of the function $f'$ is extended to $(-\infty, \infty)$, the function $f'$ works identically to the weight function $\xi$.
%%%can be associated with the following relationships: %%%$f'(z)=\xi(-z)$.}
%\begin{align*}
%\begin{split}
%f(z)=\xi(-z).
%\end{split}
%\end{align*}
In view of this relation, the function \eqref{eq:lse} associated with the log-sum-exp function yields the estimation methods that minimize $\beta$ and $\gamma$ divergences with the unnormalized and normalized estimation equations, \eqref{eq:ueq} and \eqref{eq:neq}, respectively.
%%%is corresponding to $\beta$ or $\gamma$ divergence from this relation.
In other words, when we assume the regular exponential family and the function \eqref{eq:lse}, then it is related to the estimation based on power of the statistical model.

%\sout{Comparing the arguments of weight functions $\tilde{f}'$ in \eqref{eq:bphi} and $f'$ in \eqref{eq:f-sep_eq}, when we assume that the statistical model is regular exponential family, the estimation equation \eqref{eq:f-sep_eq} of $f$-separable Bregman distortion measures can be interpreted as a weighted score function.
%The only difference is only the term $\log r_\phi(\bm{x})$.}

%Here, we ignore $\log r_\phi(\bm{x})$ and consider the relationship between \eqref{eq:bphi} and \eqref{eq:ueq}.
%We notice that domain of $f'$ and $\xi$ are $\mathbb{R}$ and $\mathbb{R}_+$ respectively, then we can associate them as follows: 
%\begin{align}
%f'(z)=\xi(-z). \label{eq:relation1}
%\end{align}
%Similarly, comparing the objective function \eqref{eq:loss_score} and \eqref{eq:neq}, we can associate $f$ and $h$ as follows:
%\begin{align}
%f(z)=h(-z). \label{eq:relation2}
%\end{align}

%\sout{The function \eqref{eq:lse} associated with the log-sum-exp function is corresponding to $\beta$ or $\gamma$ divergence from the relation between \eqref{eq:bphi} and \eqref{eq:f-sep_eq}.}

While, in this section, we have assumed the bias correction term is exactly $\bm{0}$, it does not hold in general.
%%%When we assume that the data-generating distribution is the gamma and Bregman divergence is IS distance, the estimation equation becomes unbiased without any bias correction term.
With the combination of the model and Bregman divergence discussed in the next section, the estimation equation \eqref{eq:f-sep_eq} becomes unbiased without any bias correction term for any function $f$ satisfying the condition given in the main theorems.
%%%This is in contrast to the fact that the estimation based on estimation equation needs bias correction term in general.
%\textcolor{blue}{Generally, it does not hold. 
%However, when the function $f$ is close to a linear function, the relationships \eqref{eq:relation1} and \eqref{eq:relation2} are considered to behave similarly.}
%Of course, in the case of bias correction term exactly $\bm{0}$, which will be explained in the next section, this relationship holds.
%Specifically, the difference between $f$-separable Bregman distortion measures and the estimation equation based on the weighted score function is only $\log r_\phi(\bm{x})$.
%In the cases of Mahalanobis and IS distances bias correction terms in fact vanish.
%The function $f$ can be designed from this relationship.
\section{Conditions for unbiased estimation equation}
In general, the estimator based on $f$-separable Bregman distortion measures introduced in Section \ref{sec:f-separable} does not satisfy consistency because its estimation equation is not necessarily unbiased.
%%%, a necessary condition of consistency. %%%that is, the unbiased estimation equation.
In order to satisfy an unbiased estimation equation, we must subtract the bias correction term $b_f(\bm{\theta})$ from the objective function \eqref{eq:f-separable} as follows:
\begin{align*}
&L_f(\bm{\theta}) = \frac{1}{n}\sum_{i=1}^n f\left(d_\phi\left(\bm{x}_i, \bm{\theta}\right)\right) -b_f(\bm{\theta}), \\
&b_f(\bm{\theta})= -\int  \bm{\nabla}\bm{\nabla}\phi(\bm{\theta})\mathbb{E}_{p(\bm{x}|\bm{\theta})}\left[ f'\left(d_\phi\left(\bm{x},\bm{\theta}\right)\right)\left(\bm{x}-\bm{\theta}\right)\right]d\bm{\theta},
%%%&b_f(\bm{\theta})= -\int  \bm{\nabla}\bm{\nabla}\phi(\bm{\theta})\int_{\mathbb{R}^d} p(\bm{x}|\bm{\theta})f'\left(d_\phi\left(\bm{x},\bm{\theta}\right)\right)\left(\bm{x}-\bm{\theta}\right)d\bm{x}d\bm{\theta}, \nonumber
\end{align*}
where $\int \cdot d\bm{\theta}$ denotes the indefinite integral with respect to $\bm{\theta}$.
Then, the unnormalized estimation equation is given by
\begin{align*}
%\begin{split}
%  \MoveEqLeft
  \frac{1}{n}\sum_{i=1}^n f'\left(d_\phi\left(\bm{x}_i, \bm{\theta}\right)\right) \left(\bm{x}_i-\bm{\theta}\right) = \mathbb{E}_{p(\bm{x}|\bm{\theta})}[f'\left(d_\phi\left(\bm{x},\bm{\theta}\right)\right)\left(\bm{x}-\bm{\theta}\right)] .
%\end{split}
\end{align*}
On the other hand, we can consider the normalized estimation equation as follows:
\footnotesize
\begin{align*}
%\begin{split}
%  \MoveEqLeft
\frac{\sum_{i=1}^n f'\left(d_\phi\left(\bm{x}_i, \bm{\theta}\right)\right) 
\left(\bm{x}_i-\bm{\theta}\right)}{\sum_{j=1}^n f'\left(d_\phi\left(\bm{x}_j, \bm{\theta}\right)\right) } = \frac{\mathbb{E}_{p(\bm{x}|\bm{\theta})}[f'\left(d_\phi\left(\bm{x},\bm{\theta}\right)\right)\left(\bm{x}-\bm{\theta}\right)]}{\mathbb{E}_{p(\bm{x}|\bm{\theta})}[f'\left(d_\phi\left(\bm{x},\bm{\theta}\right)\right)]} . %\nonumber
%\end{split}
\end{align*}
\normalsize
Fujisawa has elucidated that this estimation equation can possibly minimize the latent bias even when the proportion of outliers is large \cite{fujisawa2013}.
In both cases, it is necessary to calculate the integral for bias correction for each combination of statistical model, Bregman divergence, and the function $f$.
However, in many cases, the integral may not exist or be analytically intractable.
In this paper, we discuss the following estimation equation:
\begin{align*}
\frac{1}{n}\sum_{i=1}^n f'\left(d_\phi\left(\bm{x}_i, \bm{\theta}\right)\right)\left(\bm{x}_i-\bm{\theta}\right) = \bm{0} .
\end{align*}
That is, the bias correction term does not depend on the parameter $\bm{\theta}$.
In other words, the following equation is satisfied,
\begin{align}
\mathbb{E}_{p(\bm{x}|\bm{\theta})}\left[f'\left(d_\phi\left(\bm{x},\bm{\theta}\right)\right)\left(\bm{x}-\bm{\theta}\right)\right]=\bm{0}. \label{eq:bias0}
\end{align}
Then, this estimation equation is automatically normalized.
Therefore, the estimator has the possibility to minimize the latent bias even when the proportion of outliers is large.
In the rest of this section, we characterize  the combination of the statistical model $p(\bm{x}|\bm{\theta})$, Bregman divergence $d_\phi(\bm{x},\bm{\theta})$ and the function $f$ where the bias correction term vanishes.
%In the previous section, we considered the regular exponential family.
Note that the statistical model considered hereafter is generally not the regular exponential family.

In particular, we focus on Mahalanobis and IS distances.
In the case of estimating the location parameter of elliptical distribution, it is known that the bias correction term vanishes and the estimator is consistent under certain conditions on the function $f$ \cite{huberbook}.
In the case of log-gamma regression model, it is known that the bias correction term vanishes.
This is equivalent to the case where IS distance is used and the model is the gamma distribution \cite{bianco2005}.
In this paper, we derive a simple condition of the function $f$ which induces unbiased estimation equation.
In particular, in the case of IS distance, the class of the model is extended to a more general class.
%\textcolor{red}{In the case of estimating the location parameter of symmetric location family and identity $f$, it is known that the bias correction term vanishes \cite{beta-div}.
%On the other hand, in this paper, we consider general class of the function $f$.
%In the case of the IS distance, the estimator has advantages no other estimation method.}

\subsection{Mahalanobis distance}
When the strictly convex function is given by $\phi(\bm{x})=\bm{x}^{\rm T}\bm{A}\bm{x}$, where $\bm{A}$ is a positive definite matrix.
Then the corresponding Bregman divergence is given by
\begin{align*}
d_{\rm Mah.}(\bm{x}, \bm{\theta}) \triangleq (\bm{x}-\bm{\theta})^{\rm T}\bm{A}(\bm{x}-\bm{\theta}). %\label{eq:Mah}
\end{align*}
If the positive definite matrix $\bm{A}$ is identity, Mahalanobis distance reduces to squared distance,
\begin{align*}
\|\bm{x}-\bm{\theta}\|^2 = \sum_{j=1}^d ({x}^{(j)}-{\theta}^{(j)})^2 .
\end{align*}
We assume that the statistical model is the elliptical distribution.
\begin{definition}[Elliptical distribution \cite{elliptically}]
For $\bm{x}\in\mathbb{R}^d$ and the parameter $\theta\in\Theta=\mathbb{R}^d$ and the function $g: \mathbb{R}_+\to\mathbb{R}_+$, let $C<\infty$ be the normalization constant, and the positive definite matrix $\bm{A}$ be the inverse of a fixed covariance matrix.
Then the elliptical distribution is defined by the following probability density function,
\begin{align}
p(\bm{x}|\bm{\theta})=\frac{1}{C}g((\bm{x}-\bm{\theta})^{\rm T}\bm{A}(\bm{x}-\bm{\theta})). \label{eq:sym_dis}
\end{align}
\end{definition}
This distribution includes Gaussian, Laplace, $t$ distributions and so on.
%\begin{theorem}\label{thm:mahalanobis}
%\begin{enumerate}
%\item Bregman divergence used for estimation is Mahalanobis distance \eqref{eq:Mah}.
%\item The following condition holds against the combination of the function $f$ and $g$ that constitutes the statistical model \eqref{eq:sym_dis}:
%\begin{align*}
%\int_0^\infty g(t)f'(t)t^{\frac{d-1}{2}}dt <\infty .%\label{eq:thm_sq} 
%\end{align*}
%\end{enumerate}
%Then, unbiased estimation equation \eqref{eq:bias0} holds.
%\end{theorem}
\begin{theorem}\label{thm:mahalanobis}
If the following condition holds against the combination of the function $f$ and the statistical model \eqref{eq:sym_dis}, the estimation equation holds without a bias correction term:
\begin{align*}
\int_0^\infty g(t)f'(t)t^{\frac{d-1}{2}}dt <\infty .%\label{eq:thm_sq} 
\end{align*}
\end{theorem}
Although the unbiased estimation equation in this case is intuitively trivial because of the symmetry around $\bm{\theta}$ and has been pointed out in the literature \cite{huberbook}, the explicit condition for the unbiasedness has never been discussed.
%The proof of Theorem \ref{thm:mahalanobis} is in Appendix \ref{sec:proof1}.
%\begin{proof}
%\qed
%\end{proof}

%\subsubsection{Example one dimension Gaussian $g(z)=\exp(-\frac{z}{2})$}

\subsection{IS distance\label{sec:itakura}}
When the strictly convex function is given by $\phi(x)=-\log x$, then the corresponding Bregman divergence is given by
\begin{align}
d_{\rm IS}(x,\theta)\triangleq\frac{x}{\theta}-\log\frac{x}{\theta}-1. \label{eq:IS} %\label{eq:dis}
\end{align}
\begin{definition}[IS distribution]
For $x\in \mathbb{R}_+$ and the scale parameter $\theta\in\Theta=\mathbb{R}_+\setminus\{0\}$, and the function $g: \mathbb{R}_+\to\mathbb{R}_+$, we define the following probability density function with the normalization constant $C<\infty$,
\begin{align}
p(x|\theta)=\frac{1}{C}\frac{1}{x}g(d_{\rm IS}(x,\theta)). \label{eq:is_dist}
\end{align}
\end{definition}
When the expectation exists, the scale parameter also coincides with the expectation.
In particular, if $g(z)=\exp(-kz)$, the IS distribution reduces to the gamma distribution $p(x|\theta)=\left(\frac{k}{\theta}\right)^k\frac{1}{\Gamma(k)}  x^{k-1}\exp\left(-\frac{k}{\theta}x\right)$ with the known shape parameter $k>0$.
Details of the IS distribution are described in Appendix \ref{sec:is_dist}.% of the longer version of this paper~\cite{mypaper4}. %of \cite{mypaper4}.%\ref{sec:is_dist} of \cite{mypaper4}.
%\begin{theorem}\label{thm:itakura}
%\begin{enumerate}
%\item Bregman divergence used for estimation is IS distance \eqref{eq:IS}.
%\item The following condition holds against the combination of the function $f$ and $g$ that constitutes the statistical model \eqref{eq:is_dist}:
%\begin{align}
%\int_0^\infty g(t)f'(t)dt <\infty \label{eq:is_bias0}
%\end{align}
%\end{enumerate}
%Then, unbiased estimation equation \eqref{eq:bias0} holds.
%\end{theorem}
\begin{theorem}\label{thm:itakura}
If the following condition holds against the combination of the function $f$ and statistical model \eqref{eq:is_dist}, the estimation equation holds without a bias correction term:
\begin{align}
\int_0^\infty g(t)f'(t)dt <\infty \label{eq:is_bias0}
\end{align}
\end{theorem}
%The proof of Theorem \ref{thm:itakura} is omitted.%in Appendix \ref{sec:proof2}.

\subsubsection{Example: Gamma distribution}
In the case of the function \eqref{eq:lse} and gamma distribution with the known shape parameter $k>0$, that is, $g(z)=\exp(-kz)$, then the integral in \eqref{eq:is_bias0} becomes as follows:
\begin{align*}
\int_0^\infty \exp(-kz)\exp(-\alpha z)dz = \int_0^\infty \exp(-(k+\alpha)z) dz.
\end{align*}
Therefore, the condition $\alpha>-k$ must be satisfied for the integral to be bounded.
In other words, the lower limit of $\alpha$ that satisfies the unbiased estimation equation differs for each shape parameter $k$.
Since $k>0$, we can see that the condition of Theorem \ref{thm:itakura} is satisfied if $\alpha>0$, for which the estimator is robust against outliers.

In the case of the function \eqref{eq:pow} and the gamma distribution with the known shape parameter $k>0$, then the integral in \eqref{eq:is_bias0} becomes as follows:
\begin{align*}
\int_0^\infty \exp(-kz)(z+a)^{\beta-1}dz.
\end{align*}
When $a>0$, the condition of Theorem \eqref{eq:is_bias0} holds for $\beta<\infty$.
When $a=0$, the condition of Theorem \eqref{eq:is_bias0} holds for $0< \beta <\infty$.
However, it does not hold for $\beta\leq0$.

\subsection{Discussion: other Bregman divergence}
When the dimension is one, the conditions of Theorems \ref{thm:mahalanobis} and \ref{thm:itakura} are the  same.
A common point is that the statistical model is expressed by Bregman divergence used for estimation.
Hence, the results of Theorems \ref{thm:mahalanobis} and \ref{thm:itakura} can be generalized to a wider class of continuous distributions written by Bregman divergence.
We refer for the details of the continuous Bregman distribution and its theorem to Appendix \ref{sec:CBregman_dist}.% of the longer version \cite{mypaper4}.
The elliptical and IS distributions are rare examples which have unbiased estimation equations for the corresponding $f$-separable Bregman distortion measures and include the corresponding regular exponential family models. %\eqref{eq:exp_fam}. 
%\subsubsection{Extension to transformation of random variable}
%generalized gaussian, waiburu no scale
%koujigen
%\subsubsection{Example Gamma distribution $g(z)=\exp(-kz)$ under known shape parameter $k>0$}

\section{Latent bias}
In this section, we discuss the possibility of the latent bias minimization when the proportion of outliers is large.
It is induced by the vanishing bias correction term.
From the view point of the normalized estimation equation, the condition of latent bias minimization was shown as a theorem \cite{fujisawa2013}, whereas generally its condition is difficult to be examined.
However, it can be easily discussed as $\gamma$-divergence when the bias correction term vanishes.
The definitions of outliers are different for $f$-separable distortion measures and $\gamma$-divergence.
We obtain, as a by-product, a solution to a drawback of $\gamma$-divergence.

\subsection{Contaminated distribution}
We assume that the data-generating distribution is given as follows:
\begin{align*}
\tilde{p}(\bm{x})=(1-\varepsilon)p(\bm{x}|\bm{\theta}^*)+\varepsilon c(\bm{x}),
\end{align*}
where $p(\bm{x}|\bm{\theta}^*)$ is the target distribution and $c(\bm{x})$ is the contamination distribution which generates outliers and $\varepsilon$ is the proportion of outliers.
Suppose the parameter $\hat{\bm{\theta}}$ estimated from the data generated from this distribution is expressed asymptotically as $\tilde{\bm{\theta}}$.
That is, $\hat{\bm{\theta}} \xrightarrow{P} \tilde{\bm{\theta}}$.
Here, $\tilde{\bm{\theta}}-\bm{\theta}^*$ is called the latent bias, which expresses the bias caused by the contamination distribution \cite{fujisawa2013}.
%Classically, it was also called the asymptotic bias \cite{robust}.

\subsection{$\gamma$-divergence}
In the estimation based on $\gamma$-divergence, it is assumed that the following quantity can be made arbitrarily small by adjusting $\gamma_0> 0$ as an assumption regarding outliers,
\begin{align}
\nu_{p} = \left[ \mathbb{E}_{c(\bm{x})}\left[p(\bm{x}|\bm{\theta}^*)^{\gamma_0}\right]\right]^\frac{1}{\gamma_0}. \label{eq:gamma-assump}
\end{align}
This assumption means that outliers are distributed over the region where the likelihood is small in the target distribution $p(\bm{x}|\bm{\theta}^*)$.
Since nothing about the outlier proportion is assumed, it is also possible to deal with the case where the outlier proportion is large.
Kuchibhotla et al. reported $\gamma$-divergence is adversely affected by data at the edge of the support of the target model \cite{bridge-div}.
For example, in the estimation of the scale parameter of the exponential distribution, a wrong global solution is generated when very small inlier around $x=0$ such as $x=10^{-4}$ is mixed.
Recently, a solution to this problem has been invented, whereas it is not fully resolved \cite{bridge-div}.

\subsection{$f$-separable Bregman distortion measures}
In the estimation based on $f$-separable Bregman distortion measures, we assume that the following quantity can be made arbitrarily small by adjusting the function $f$ as an assumption regarding outliers,
\begin{align}
\nu_{d_\phi} =\mathbb{E}_{c(\bm{x})}\left[f\left(d_\phi\left(\bm{x}, \bm{\theta}^*\right)\right) \right], \label{eq:f-assump}
\end{align}
under Assumption \ref{ass:f} described later.
This assumption is corresponding to the assumption \eqref{eq:gamma-assump} of $\gamma$-divergence and means that when the random variable follows the contamination distribution, that is, $\bm{X} \sim c(\bm{X})$, an outlier is in the region where $d_\phi(\bm{X},\bm{\theta}^*)\to\infty$ is satisfied.
When estimating the location parameter of the elliptical distribution using Mahalanobis distance, the definition of outlier is same as \eqref{eq:gamma-assump}.
That is, $\bm{x}$ with $\|\bm{x}\|\to\infty$ is regarded as the outlier.
However, when estimating the scale parameter of the IS distribution using IS distance, the definition of outlier is not same as \eqref{eq:gamma-assump}.
In this case, from
\begin{align*}
\lim_{x\to0}d_{\rm IS}(x,\theta)=\lim_{x\to\infty}d_{\rm IS}(x,\theta)=\infty,
\end{align*}
the data near $0$ or $\infty$ are regarded as outliers.
In other words, the estimator based on $f$-separable IS distortion measures is robust against large outliers and very small inliers to which $\gamma$-divergence is vulnerable.

\subsection{Condition of function $f$}
In the following, we identify the function $f(z)$ with $f(z)+ {\rm constant}$, because the estimator depends only on the derivative of the function $f$.
\begin{assumption}\label{ass:f}
%\begin{align*}
$\forall z\in \mathbb{R}_+, \; |f(z)| < \infty$
%\\
and 
$\displaystyle \lim_{z\to\infty}f(z)=0$
%\end{align*}
\end{assumption}
\begin{assumption} \label{ass:select_f}
Under Assumption \ref{ass:f}, \eqref{eq:f-assump} can be made arbitrarily small by adjusting the function $f$.
\end{assumption}
\begin{assumption} \label{ass:cons}
When $\epsilon=0$, the estimator $\tilde{\bm{\theta}}$ is a consistent estimator, that is, $\tilde{\bm{\theta}}=\bm{\theta}^*$.
\end{assumption}
\begin{theorem}\label{thm:latent_bias}
Under Assumptions \ref{ass:f}-\ref{ass:cons}, the latent bias can be made arbitrarily small by adjusting  the function $f$.
\end{theorem}
%The proof of Theorem \ref{thm:latent_bias} is in Appendix \ref{sec:proof3}.

\subsubsection{Example}
We consider the function \eqref{eq:lse}, which is identified with $\exp(-\alpha z)$.
Then Assumption \ref{ass:f} holds immediately.
Assumption \ref{ass:select_f} follows from Lyapunov's inequality with sufficiently large $\alpha$.
Assumption \ref{ass:cons} depends on the target distribution.
In the case of the gamma distribution, we can prove the consistency of the estimator \cite{deviance-based}.
%%%%%%%%%%!!!!!Proof to be added.!!!!%%%%%%%%% 
%Details is put in Appendix xxx.}

%We consider the function \eqref{eq:lse}
%Assumption \ref{ass:select_f} holds immediately.
%Assumption \ref{ass:cons} holds from Lyapunov's inquality.
%Details is puts in Appendix xxx.

%\subsection{Discussion}
\section{Asymptotic property}
The estimation based on $f$-separable Bregman distortion measures, which satisfies the unbiasedness of estimation equation, can be interpreted as an M-estimation.
Therefore, under appropriate assumptions, the following consistency and asymptotic normality of the estimator follow from the asymptotic theory of M-estimation \cite{robust,huberbook,asympt},
\begin{align*}
\hat{\bm{\theta}}\xrightarrow{P}\bm{\theta}^*,
%\end{align}
%\begin{align}
\sqrt{n}\left(\hat{\bm{\theta}}-\bm{\theta}^*\right)\xrightarrow{d}N(\bm{0}, \Sigma(\bm{\theta}^*)),
\end{align*}
where $\Sigma(\bm{\theta}^*)=\bm{J}^{-1}(\bm{\theta}^*)\bm{I}(\bm{\theta}^*)\bm{J}^{-1}(\bm{\theta}^*)$,
\begin{align*}
&\bm{I}(\bm{\theta})=\mathbb{E}_{p(\bm{x}|\bm{\theta})}\left[[f'\left(d_\phi\left(\bm{x},\bm{\theta}\right)\right)]^2\left(\bm{x}-\bm{\theta}\right)\left(\bm{x}-\bm{\theta}\right)^{\rm T}\right], \\
%\end{align}
%\begin{align}
&\bm{J}(\bm{\theta})=\mathbb{E}_{p(\bm{x}|\bm{\theta})}\left[\frac{\partial f'\left(d_\phi\left(\bm{x},\bm{\theta}\right)\right)\left(\bm{x}-\bm{\theta}\right)}{\partial \bm{\theta}}\right].
\end{align*}
If the proportion of outliers is large, the asymptotic variance is given by the technique in \cite{gamma-div}, \cite{fujisawa2013}.

\subsection{Gamma distribution}
We assume that the statistical model is the gamma distribution $p(x|\theta)=\left(\frac{k}{\theta}\right)^k\frac{1}{\Gamma(k)}  x^{k-1}\exp\left(-\frac{k}{\theta}x\right)$, the function $f$ is \eqref{eq:lse} and Bregman divergence is IS distance \eqref{eq:IS}, then the asymptotic variance of the estimator is given by
\begin{align*}
V[\hat{\theta}]=\Sigma(\theta^{*})=\frac{\Gamma(2\alpha+k)\Gamma(k)}{[\Gamma(\alpha+k)]^2}\frac{(\alpha+k)^{2(\alpha+1+k)}}{(2\alpha+k)^{2\alpha+1+k}}\frac{1}{k^{2+k}}\theta^{*2},
\end{align*}
where $\Gamma(\cdot)$ is the gamma function and the tuning parameter satisfies $\alpha >-0.5k$.
In the case of the exponential distribution ($k=1$), we can compare the asymptotic relative efficiencies (AREs) of the estimators based on minimizing $f$-separable IS distortion measures and $\beta$ and $\gamma$ divergences.
The ARE is given by $\frac{V[\hat{\theta}_{\rm MLE}]}{V[\hat{\theta}]}$,
%\begin{align*}
%\frac{V[\hat{\theta}_{\rm MLE}]}{V[\hat{\theta}]},
%\end{align*}
where $V[\hat{\theta}_{\rm MLE}]$ is the asymptotic variance of the maximum likelihood estimator ($\alpha=0$).
In the case of the exponential distribution, the asymptotic variance of the estimators based on $\beta$ and $\gamma$ divergence were derived respectively \cite{beta-div}, \cite{type0-div}.
Figure \ref{fig:ARE} shows their AREs, when the tuning parameter $\alpha = \beta = \gamma$.
We notice that the range of tuning parameter $\alpha = \beta = \gamma >0$ induces the robustness against outliers.
From Figure \ref{fig:ARE}, for the function \eqref{eq:lse} and IS distance, the ARE is generally greater than that of $\beta$-divergence in the range of tuning parameter $\alpha<2$.
The ARE is also greater than that of $\gamma$-divergence in the entire range of the tuning parameter.
However, in general, the ARE and robustness have trade-off relationship.
Hence, it is important to choose the tuning parameter appropriately taking into account both of them.

\begin{figure}[tb]
\begin{center}
  \includegraphics[width=0.4\textwidth]{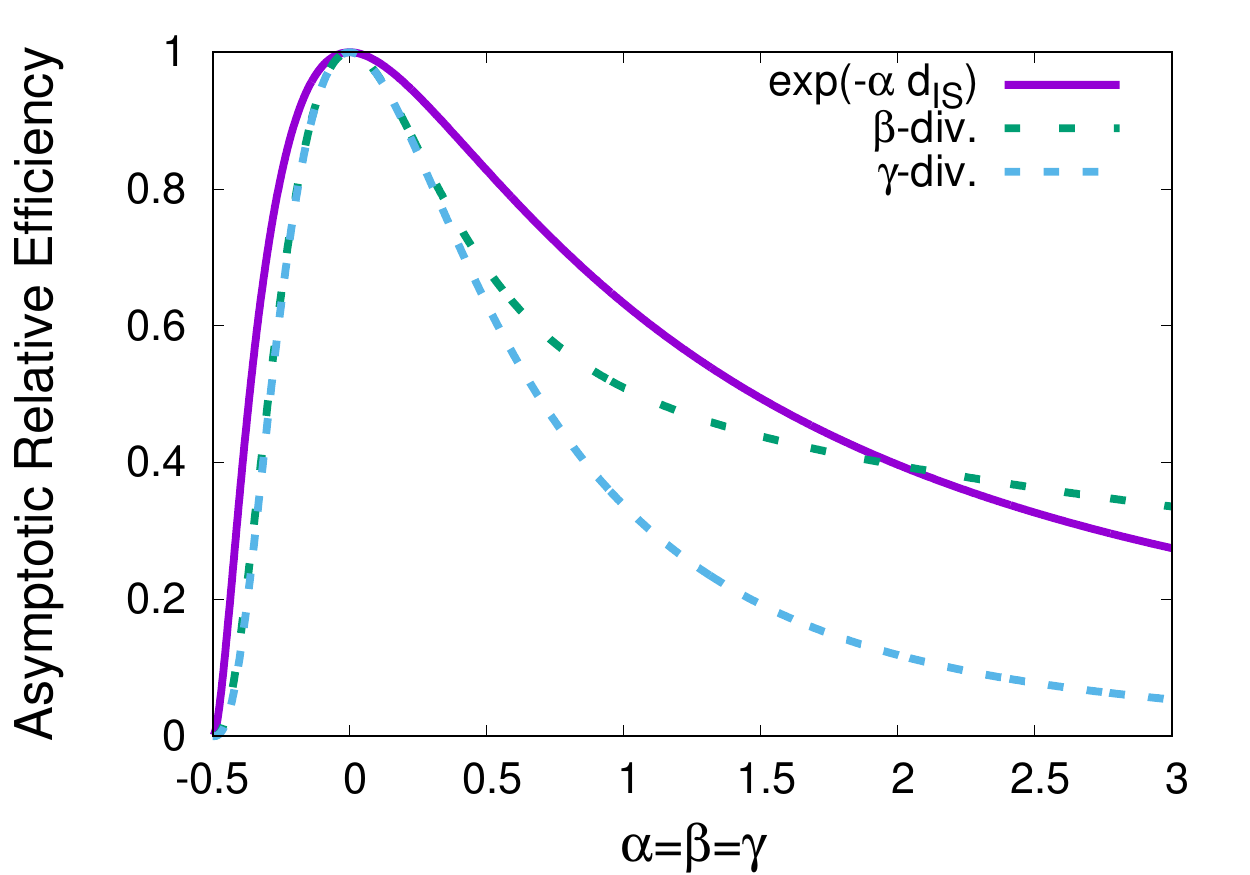}
\end{center}
\caption{Comparison of asymptotic relative efficiency under the exponential model ($k=1$).}
%\ecaption{}
\label{fig:ARE}
\vspace{-10pt}
\end{figure}
%\input{./7_Regression_case}
%\input{./8_Experiment}
%\input{./9_Discussion}
%\section{Discussion}
\section{Conclusion}
In this paper, we discussed the condition for the unbiased estimation equation in the class of parameter estimation by minimizing $f$-separable Bregman distortion measures.
Its condition consists of the statistical model, Bregman divergence and the function $f$.
We clarified in the cases of Mahalanobis and IS distances that the condition the function $f$ and the statistical model should satisfy is characterized by a simple integral.
In the parameter estimation of the scale parameter of the gamma distribution, divergence-based estimation generally requires bias correction terms.
Furthermore, we proved that the vanishing bias correction term implies the possibility of minimizing latent bias caused by the large proportion of outliers.
%%%\input{./9_Appendix}
%\bibliographystyle{iclr2020_conference.bst} 
%\bibliography{Refj.bib}

%%%%%%
%% Appendix:
%% If needed a single appendix is created by
%%
%\appendix
%%
%% If several appendices are needed, then the command
%%
% \appendices
%%
%% in combination with further \section-commands can be used.
%%%%%%

%\section*{Acknowledgment}
%This work was supported by JSPS KAKENHI Grant numbers 19J11776 and 19K11825.

%%%%%%
%% To balance the columns at the last page of the paper use this
%% command:
%%
%\enlargethispage{-1.2cm} 
%%
%% If the balancing should occur in the middle of the references, use
%% the following trigger:
%%
%\IEEEtriggeratref{3}
%%
%% which triggers a \newpage (i.e., new column) just before the given
%% reference number. Note that you need to adapt this if you modify
%% the paper.  The "triggered" command can be changed if desired:
%%
%\IEEEtriggercmd{\enlargethispage{-20cm}}
%%
%%%%%%

%%%%%%
%% References:
%% We recommend the usage of BibTeX:
%%
\bibliographystyle{myIEEEbib}
\bibliography{Refj}
%%
%% where we here have assume the existence of the files
%% definitions.bib and bibliofile.bib.
%% BibTeX documentation can be obtained at:
%% http://www.ctan.org/tex-archive/biblio/bibtex/contrib/doc/
%%%%%%

%% Or you use manual references (pay attention to consistency and the
%% formatting style!):

\appendix
%\onecolumn
\section{Proof of Theorems}
\subsection{Proof of Theorem \ref{thm:mahalanobis} \label{sec:proof1}}
We assume that eigenvalue decomposition with respect to positive definite matrix $\bm{A}$.
That is, $\bm{A}=\bm{V}\bm{\Lambda}\bm{V}^{\rm T}$, where $\bm{V}^{-1}=\bm{V}^{\rm T}$ and $\bm{\Lambda}$ is the diagonal matrix with eigenvalues.
Then, Mahalanobis distance is rewritten by
\begin{align*}
\begin{split}
&\left(\bm{x}-\bm{\theta}\right)^{\rm T}\bm{A}\left(\bm{x}-\bm{\theta}\right) 
\\={}&\left(\bm{x}-\bm{\theta}\right)^{\rm T}\bm{V}\bm{\Lambda}\bm{V}^{\rm T}\left(\bm{x}-\bm{\theta}\right) 
\\={}&\bm{y}^{\rm T}\bm{\Lambda}\bm{y} 
=\sum_{j=1}^d \lambda_j y_j^2,
\end{split}
\end{align*}
where $\bm{y}=\bm{V}^{\rm T}\left(\bm{x}-\bm{\theta}\right)$ and $\lambda_j$ is the $j$-th element of the diagonal matrix $\bm{\Lambda}$.
Further, we assume that rank factorization with respect to positive definite matrix $\bm{\Lambda}$.
That is, $\bm{\Lambda}=\sqrt{\bm{\Lambda}}^{\rm T} \sqrt{\bm{\Lambda}}$.
If the random vector $\bm{Y}$ follow  $\bm{Y}\sim\frac{1}{C}g(\bm{Y}^{\rm T}\bm{A}\bm{Y})$, then it can be decomposed as $\bm{Y}=R\bm{U}\sqrt{\bm{\Lambda}}$,
where random variable $R$ satisfies $R\geq 0$ and $d$-dimensional random vector $\bm{U}$ is uniformly distributed on the unit sphere surface \cite{elliptically}.
Then, $\mathbb{E}[\bm{U}]=\bm{0}$ holds.
From \eqref{eq:bias0}, ignoring the normalization constant $C$, we have
\begin{align*}
\begin{split}
&\int_{\mathbb{R}^d} g\left(d_{\rm Mah.}(\bm{x}, \bm{\theta})  \right)f'\left(d_{\rm Mah.}(\bm{x}, \bm{\theta})  \right)(\bm{x}-\bm{\theta})d\bm{x}
%\\={}&\int_{\mathbb{R}^d} g\left(\tilde{\bm{x}}^{\rm T}\bm{A}\tilde{\bm{x}})  \right)f'\left(\tilde{\bm{x}}^{\rm T}\bm{A}\tilde{\bm{x}})  \right)\tilde{\bm{x}}d\tilde{\bm{x}}
\\={}&\int_{\mathbb{R}^d} g\left(\sum_{j=1}^d \lambda_j y_j^2\right)f'\left(\sum_{j=1}^d \lambda_j y_j^2\right)\bm{V}\bm{y}\abs{\bm{V}}d\bm{y}
\\={}&\abs{\bm{V}}\bm{V} \left[\prod_{j=1}^d \frac{1}{\sqrt{\lambda_j}}\right] \underbrace{\mathbb{E}[\bm{U}]}_{\bm{0}}\int_0^{\infty} g(r^2)f'(r^2)r^d dr 
\\={}&\bm{0}.
\end{split}
\end{align*}
Therefore, if the following integral exists, then the unbiased estimation equation holds without any bias correction term
\begin{align*}
\int_0^{\infty} g(r^2)f'(r^2)r^d dr \\
=\int_0^{\infty} g(t)f'(t)t^{\frac{d-1}{2}} dt,
\end{align*}
where we used integration by substitution as $t=r^2$.
\qed

\subsection{Proof of Theorem \ref{thm:itakura} \label{sec:proof2}}
From \eqref{eq:bias0}, ignoring the normalization constant $C$, we have
\begin{align*}
\begin{split}
&\int_0^{\infty} \frac{1}{x}g(d_{\rm IS}(x,\theta))f'(d_{\rm IS}(x,\theta))(x-\theta)dx
\\={}&\int_0^\theta \frac{1}{x}g(d_{\rm IS}(x,\theta))f'(d_{\rm IS}(x,\theta))(x-\theta)dx
\\&+\int_\theta^{\infty} \frac{1}{x}g(d_{\rm IS}(x,\theta))f'(d_{\rm IS}(x,\theta))(x-\theta)dx
\\={}&\theta \int_\infty^0 g(t)f'(t)dt + \theta \int_0^\infty g(t)f'(t)dt=0.
\end{split}
\end{align*}
We used integration by substitution as $t=d_{\rm IS}(x,\theta)$.
Therefore, if the following integral exists, then the unbiased estimation equation holds without any bias correction term
\begin{align*}
\int_0^\infty g(t)f'(t)dt <\infty. %\label{eq:is_bias0}
\end{align*}
\qed

\subsection{Proof of Theorem \ref{thm:latent_bias} \label{sec:proof3}}
We take the expectation of the objective function \eqref{eq:f-separable} by $\tilde{p}(\bm{x})=(1-\varepsilon)p(\bm{x}|\bm{\theta}^*)+\varepsilon c(\bm{x})$ as $\int \tilde{p}(\bm{x})f\left(d_\phi\left(\bm{x}, \bm{\theta}\right)\right)d\bm{x}$.
We have
\begin{align*}
\begin{split}
&\int \tilde{p}(\bm{x})f\left(d_\phi\left(\bm{x}, \bm{\theta}\right)\right)d\bm{x}
\\={}&(1-\varepsilon)\int p(\bm{x}|\bm{\theta}^*)f\left(d_\phi\left(\bm{x}, \bm{\theta}\right)\right)d\bm{x}+
\varepsilon \int  c(\bm{x})f\left(d_\phi\left(\bm{x}, \bm{\theta}\right)\right)d\bm{x}
\\={}&(1-\varepsilon)\int p(\bm{x}|\bm{\theta}^*)f\left(d_\phi\left(\bm{x}, \bm{\theta}\right)\right)d\bm{x}+O(\varepsilon\nu_{d_\phi}).
%\\={}&(1-\varepsilon)\int p(\bm{x}|\bm{\theta}^*)f\left(d_\phi\left(\bm{x}, \bm{\theta}\right)\right)d\bm{x}+O(\varepsilon \nu_p)
\end{split}
\end{align*}
Here, we consider $\varepsilon\nu_{d_\phi}= \int  c(\bm{x})f\left(d_\phi\left(\bm{x}, \bm{\theta}\right)\right)d\bm{x}$ as $\bm{\theta}\approx\bm{\theta}^*$.
From Assumptions \ref{ass:f}, \ref{ass:select_f}, we can ignore $O(\varepsilon\nu_{d_\phi})$,
\begin{align*}
\begin{split}
\tilde{\bm{\theta}}&=\argmin_{\bm{\theta}}\int \tilde{p}(\bm{x})f\left(d_\phi\left(\bm{x}, \bm{\theta}\right)\right)d\bm{x} \\
&= \argmin_{\bm{\theta}}\int p(\bm{x}|\bm{\theta}^*)f\left(d_\phi\left(\bm{x}, \bm{\theta}\right)\right)d\bm{x}=\bm{\theta}^* ,
\end{split}
\end{align*}
where we have used Assumption \ref{ass:cons}.
Therefore, the latent bias can be made sufficiently small, that is, $\tilde{\bm{\theta}}-\bm{\theta}^* \approx \bm{0}$ by adjusting the function $f$.
\qed

\subsection{Detail of IS distribution \label{sec:is_dist}}
In Section \ref{sec:itakura}, we defined a new distribution and named the IS distribution because it is characterized by the IS distance.
In this appendix, we explain properties of IS distribution.
IS distribution is defined by
\begin{align*}
p(x|\theta)=\frac{1}{C}\frac{1}{x}g(d_{\rm IS}(x,\theta)),
\end{align*}
where $C$ is the normalization constant.
The normalization constant $C$ is expressed as follows without depending on $\theta$,
\begin{align*}
C=\int_{0}^\infty \frac{1}{x}g(d_{\rm IS}(x,\theta))dx
=\int_{0}^\infty \frac{1}{t}g(d_{\rm IS}(t,1))dt.
\end{align*}
We used integration by substitution $t=x/\theta$.
Specifically, if the expected value exists, $\mathbb{E}[X]<\infty$, $\mathbb{E}[X]=\theta$ holds from the estimation equation \eqref{eq:bias0} with $f(z)=z$.
The condition for the unbiased estimation equation is given by \eqref{eq:is_bias0}.
This condition with $f(z)=z$ reduces to $\int_{0}^\infty g(t)dt<\infty$, that is, $g\in L^1(\mathbb{R}_+)$.
Thus, the following relation holds with respect to the expectation and the function $g$,
\begin{align}
    g\in L^1(\mathbb{R}_+) \Leftrightarrow \mathbb{E}[X]=\theta. \label{eq:g_expect}
\end{align}
In other words, the existence of the expectation depends only on the function $g$.
This property holds in the general continuous Bregman distribution described later.
Then, the normalization constant is expressed as 
\begin{align}
C=\int_{0}^\infty g(d_{\rm IS}(x,1))dx. \label{eq:const}
\end{align}
Because we have
\begin{align*}
\begin{split}
&\theta=\mathbb{E}[X]=\int_0^\infty \frac{1}{C}\frac{1}{x}g(d_{\rm IS}(x,\theta))x dx\\
{}&=\frac{1}{C}\int_0^\infty g(d_{\rm IS}(x,\theta)) dx=\theta \frac{1}{C}\int_0^\infty g(d_{\rm IS}(x,1)) dx,
\end{split}
\end{align*}
the normalization constant $C$ must satisfy \eqref{eq:const}.
\subsubsection{Example: Gamma distribution}
When we choose the function $g(z)=\exp(-kz)$, IS distribution becomes the gamma distribution with the known shape parameter $k>0$.
Then, $\frac{1}{x}g(d_{\rm IS}(x,\theta))$ is expressed as
\begin{align*}
\begin{split}
&\frac{1}{x}g(d_{\rm IS}(x,\theta)) = \frac{1}{x}\exp\left(-kd_{\rm IS}(x,\theta)\right)
\\={}&\frac{1}{x}\exp(-\frac{k}{\theta}x)\left(\frac{{e}}{\theta}\right)^k x^k 
=\left(\frac{{e}}{\theta}\right)^k x^{k-1} \exp\left(-\frac{k}{\theta}x\right).
\end{split}
\end{align*}
The normalization constant $C$ is given by
\begin{align*}
\begin{split}
&C = \int_0^\infty \frac{1}{x}\exp\left(-kd_{\rm IS}(x,\theta)\right) dx
\\={}&\left(\frac{{e}}{\theta}\right)^k \int_0^\infty x^{k-1} \exp\left(-\frac{k}{\theta}x\right)dx
\\={}&\left(\frac{{e}}{\theta}\right)^k \left(\frac{\theta}{k}\right)^k \Gamma(k)
\\={}&\left(\frac{{e}}{k}\right) \Gamma(k),
\end{split}
\end{align*}
where $\Gamma(\cdot)$ is the gamma function.
Therefore, the gamma distribution is obtained
\begin{align}
\begin{split}
&p(x|\theta) = \frac{1}{C}\frac{1}{x}\exp(-kd_{\rm IS}(x,\theta))
\\={}&\left(\frac{k}{{e}}\right)^k\frac{1}{\Gamma(k)} \left(\frac{{e}}{\theta}\right)^k x^{k-1} \exp\left(-\frac{k}{\theta}x\right)
\\={}&\left(\frac{k}{\theta}\right)^k\frac{1}{\Gamma(k)}  x^{k-1} \exp\left(-\frac{k}{\theta}x\right). \label{eq:gamma_dis}
\end{split}
\end{align}
The gamma distribution is also expressed as
\begin{align*}
p(x|\beta, k)=\frac{x^{k-1}}{\Gamma(k)\beta^k}\exp\left(-\frac{x}{\beta}\right).
\end{align*}
The parameters $\beta$ and $k$ are called scale and shape parameters, respectively.
This model is corresponding to \eqref{eq:gamma_dis} by the transformation $\theta = k\beta$.
Notice that the parameter $\theta$ is also the scale parameter and the expectation parameter.

\subsection{Detail of Continuous Bregman distribution \label{sec:CBregman_dist}}
\begin{definition}[Continuous Bregman distribution]
For $x\in (a,b)\subseteq \mathbb{R}$, the parameter $\theta\in\Theta\subseteq\mathbb{R}$, and the function $g: \mathbb{R}_+\to\mathbb{R}_+$, we define the following probability density function with the normalization constant satisfying $C(\theta)<\infty$,
\begin{align}
p(x|\theta)=\frac{1}{C(\theta)}\frac{\phi'(x)-\phi'(\theta)}{x-\theta}g(d_{\phi}(x,\theta)). \label{eq:CBregman_dist}
\end{align}
\end{definition}
Specifically, if \eqref{eq:cond1} holds and the expected value exists, $\mathbb{E}[X]<\infty$, $\mathbb{E}[X]=\theta$ holds from the estimation equation \eqref{eq:bias0} with $f(z)=z$ and the condition for it is given by \eqref{eq:CBregman_bias0}.
For the same reason, the relationship \eqref{eq:g_expect} holds for the expectation and the function $g$ as in the case of the IS distribution.
Note that the existence of the expectation depends only on the function $g$ regardless of the choice of Bregman divergence as long as the normalization constant $C(\theta)$ exists and \eqref{eq:cond1} holds.
Note that in general, the normalization constant $C(\theta)$ depends on the parameter $\theta$.
\begin{assumption}
\begin{enumerate}
\item Bregman divergence satisfies the following for any $\theta$ and a positive constant $\zeta$ (including $\infty$):
\begin{align}
\lim_{x\to a}d_\phi(x,\theta)=\lim_{x\to b}d_\phi(x,\theta)=\zeta. \label{eq:cond1}
\end{align}
\item Bregman divergence used for estimation is corresponding to that of the model \eqref{eq:CBregman_dist}.
%%%\item The following condition holds against the combination of the function $f$ and $g$ that constitutes the statistical model \eqref{eq:CBregman_dist}:
%%%\begin{align*}
%%%\int_0^\infty g(t)f'(t)dt <\infty. %\label{eq:is_bias0}
%%%\end{align*}
\end{enumerate}
\end{assumption}
Under these assumptions, the unbiased estimation equation \eqref{eq:bias0} holds.
\begin{theorem}
\label{thm:CBregman}
If the following condition holds against the combination of the function $f$ and statistical model \eqref{eq:CBregman_dist}, the estimation equation holds without a bias correction term:
\begin{align}
\int_0^\infty g(t)f'(t)dt <\infty. \label{eq:CBregman_bias0}
\end{align}
\end{theorem}
%We assume $\lim_{x\to a}d_\phi(x,\theta)=\lim_{x\to b}d_\phi(x,\theta)=\infty$.
\begin{proof}
From \eqref{eq:bias0}, ignoring the normalization constant $C(\theta)$, we have
\begin{align*}
\begin{split}
&\int_\mathbb{R} \frac{\phi'(x)-\phi'(\theta)}{x-\theta}g(d_{\phi}(x,\theta))f'(d_{\phi}(x,\theta)) (x-\theta) dx
\\={}&\int_a^\theta {(\phi'(x)-\phi'(\theta))}g(d_{\phi}(x,\theta))f'(d_{\phi}(x,\theta)) dx
\\& + \int_\theta^b {(\phi'(x)-\phi'(\theta))}g(d_{\phi}(x,\theta))f'(d_{\phi}(x,\theta)) dx
\\={}&\int_{\zeta}^0 g(t)f'(t) dt + \int_0^\zeta g(t)f'(t) dt = 0.
\end{split}
\end{align*}
We used integration by substitution as $t=d_\phi(x,\theta)$ and \eqref{eq:cond1}.
Therefore, if integral \eqref{eq:CBregman_bias0} exists, then the unbiased estimation equation holds without any bias correction term.
\qed
\end{proof}
The following models are the examples of the continuous Bregman distribution.
\subsubsection{symmetric (one dimensional elliptical) distribution}
We set $\phi(x)=x^2$. Then, \eqref{eq:CBregman_dist} becomes the symmetric (one dimensional elliptical) distribution as follows:
\begin{align*}
p(x|\theta) = \frac{1}{C}g((x-\theta)^2).
\end{align*}
\subsubsection{IS distribution}
We set $\phi(x)=-\log x$. Then, \eqref{eq:CBregman_dist} becomes the IS distribution as follows:
\begin{align*}
p(x|\theta)=\frac{1}{C}\frac{1}{x}g(d_{\rm IS}(x,\theta)).
\end{align*}

%%%\subsection{The relation of continuous Bregman distribution and regular exponential family}
\begin{comment}
\begin{align}
g(z)=\exp(-z),\\
\forall x, \frac{1}{C(\theta)}\frac{\phi'(x)-\phi'(\theta)}{x-\theta}
\end{align}
special case: gauss, gamma
\subsubsection{Special case: subclass of one dimensional regular exponential family}
\end{comment}

Finally, we consider the relation between the continuous Bregman distribution \eqref{eq:CBregman_dist} and the regular exponential family \eqref{eq:exp_fam}.
Let $g(z)=\exp(-z)$. 
If the factor
\[
\frac{1}{C(\theta)}\frac{\phi'(x)-\phi'(\theta)}{x-\theta}
\]
does not depend on the parameter $\theta$,
\eqref{eq:CBregman_dist} becomes one dimensional regular exponential family as follows:
\begin{align*}
p(x|\theta)=r_{\phi}(x)\exp(-d_{\phi}(x, \theta)),
\end{align*}
where $r_\phi(x)$ is uniquely determined by the strictly convex function $\phi$ \cite{clustering-bregman}.
The Gaussian and gamma distributions provide rare examples included in both the class of continuous Bregman distributions and the regular exponential family.

\end{document}